\documentclass[11pt,a4paper]{article}
\usepackage[hyperref]{acl2019}
\usepackage{times}
\usepackage{latexsym}
\usepackage{url}
\usepackage{multirow}
\usepackage{amsmath}
\usepackage{amsfonts}
\usepackage{graphicx}
\usepackage{threeparttable} 
\usepackage{booktabs}  
\usepackage{amsthm}
\usepackage{todonotes}
\usepackage{environ}
\usepackage{diagbox}
\usepackage{subfig}
\usepackage{tikz}
\usepackage{pgfplots}
\usepackage{footnote}
\makesavenoteenv{tabular}
\makesavenoteenv{table}
\allowdisplaybreaks
\NewEnviron{sealign}{\small\begin{align} \BODY \end{align}}
\NewEnviron{seequation}{\begin{equation}\small \BODY \end{equation}}

\newtheorem{prop}{Property}

\aclfinalcopy 

\title{TENER: Adapting Transformer Encoder for Named Entity Recognition}

\author{Hang Yan, Bocao Deng, Xiaonan Li, Xipeng Qiu$^*$ \\\\
	School of Computer Science, Fudan University\\ Shanghai Key Laboratory of Intelligent Information Processing, Fudan University \\
	{\small \texttt{\{hyan19,xpqiu\}@fudan.edu.cn, \{dengbocao,lixiaonan1208\}@gmail.com} }
}

\date{}

\begin{document}
\maketitle

\renewcommand{\thefootnote}{\fnsymbol{footnote}}
\footnotetext[1]{Corresponding author.}
\renewcommand{\thefootnote}{\arabic{footnote}}

\begin{abstract}
 Bidirectional long short-term memory networks (BiLSTMs) have been widely used as an encoder for named entity recognition (NER) task.  Recently, the fully-connected self-attention architecture (aka Transformer) is broadly adopted in various natural language processing (NLP) tasks owing to its parallelism and advantage in modeling the long-range context. Nevertheless, the performance of the vanilla Transformer in NER is not as good as it is in other NLP tasks. In this paper, we propose TENER, a NER architecture adopting adapted Transformer Encoder to model the character-level features and word-level features. By incorporating the direction-aware, distance-aware and un-scaled attention, we prove the Transformer-like encoder is just as effective for NER as other NLP tasks.
 Experiments on six NER datasets show that TENER achieves superior performance than the prevailing BiLSTM-based models.
\end{abstract}

\section{Introduction}

The named entity recognition (NER) is the task of finding the start and end of an entity in a sentence and assigning a class for this entity. NER has been widely studied in the field of natural language processing (NLP) because of its potential assistance in question generation~\cite{DBLP:conf/nlpcc/ZhouYWTBZ17}, relation extraction~\cite{miwa2016end}, and coreference resolution~\cite{fragkou2017applying}. Since~\cite{collobert2011natural}, various neural models have been introduced to avoid hand-crafted features~\cite{huang2015bidirectional,ma2016end,lample2016neural}.

NER is usually viewed as a sequence labeling task, the neural models usually contain three components: word embedding layer, context encoder layer, and decoder layer~\cite{huang2015bidirectional,ma2016end,lample2016neural,DBLP:journals/tacl/ChiuN16,chen2019grn,DBLP:conf/acl/ZhangLS18,gui2019lexicon}. The difference between various NER models mainly lies in the variance in these components.

Recurrent Neural Networks (RNNs) are widely employed in NLP tasks due to its sequential characteristic, which is aligned well with language. Specifically, bidirectional long short-term memory networks (BiLSTM)~\cite{hochreiter1997long} is one of the most widely used RNN structures. \citep{huang2015bidirectional} was the first one to apply the BiLSTM and Conditional Random Fields (CRF)~\cite{lafferty2001conditional} to sequence labeling tasks. Owing to BiLSTM's high power to learn the contextual representation of words, it has been adopted by the majority of NER models as the encoder~\cite{ma2016end,lample2016neural,DBLP:conf/acl/ZhangLS18,gui2019lexicon}.

Recently, Transformer~\cite{vaswani2017attention} began to prevail in various NLP tasks, like machine translation~\cite{vaswani2017attention}, language modeling~\cite{radford2018improving}, and pretraining models~\cite{DBLP:journals/corr/abs-1810-04805}.
The Transformer encoder adopts a fully-connected self-attention structure to model the long-range context, which is the weakness of RNNs.
Moreover, Transformer has better parallelism ability than RNNs.
However, in the NER task, Transformer encoder has been reported to perform poorly~\cite{DBLP:conf/naacl/GuoQLSXZ19}, our experiments also confirm this result.
Therefore, it is intriguing to explore the reason why Transformer does not work well in NER task.

In this paper, we analyze the properties of Transformer and propose two specific improvements for NER.

%

\begin{figure}[!htb]
    \centering
      \resizebox {0.9\linewidth} {!} {
      \begin{tikzpicture}[font=\small,every node/.style={inner sep=1,outer sep=1,minimum height=1em, text depth=0cm}]
      \node[text=red!70,draw=black](a1){Louis};
      \node[text=red!70,draw=black, xshift=2.7em](a2){Vuitton};
      \node[xshift=5.8em](b){founded};
      \node[text=blue!70,draw=black,xshift=8.6em](c1){Louis};
      \node[text=blue!70,draw=black,xshift=11.3em](c2){Vuitton};
      \node[text=blue!70,draw=black,xshift=13.65em](c3){Inc.};
      \node[xshift=15.0em](d){in};
      \node[text=green!70,draw=black,xshift=16.5em](e){1854};
      \node[xshift=17.7em](f){.};


       \node[below of=a1,node distance =1em,]{\texttt{PER}};
       \node[below of=a2,node distance =1em,]{\texttt{PER}};
       \node[below of=c1,node distance =1em,]{\texttt{ORG}};
       \node[below of=c2,node distance =1em,]{\texttt{ORG}};
       \node[below of=c3,node distance =1em,]{\texttt{ORG}};
       \node[below of=e,node distance =1em,]{\texttt{TIME}};
      \end{tikzpicture}
      }
    \caption{An example for NER. The relative direction is important in the NER task, because words before ``Inc." are mostly to be an organization, words after ``in" are more likely to be time or location. Besides, the distance between words is also important, since only continuous words can form an entity, the former ``Louis Vuitton" can not form an entity with the ``Inc.". }
    \label{fig:example}
\end{figure}

The first is that the sinusoidal position embedding used in the vanilla Transformer is aware of distance  but unaware of the directionality. In addition, this property will lose when used in the vanilla Transformer. However, both the direction and distance information are important in the NER task. For example in Fig \ref{fig:example}, words after ``in" are more likely to be a location or time than words before it, and words before ``Inc." are mostly likely to be of the entity type ``ORG". Besides, an entity is a continuous span of words. Therefore, the awareness of distance might help the word better recognizes its neighbor. To endow the Transformer
with the ability of direction- and distance-awareness, we adopt the relative positional encoding~\cite{DBLP:conf/naacl/ShawUV18,DBLP:conf/iclr/HuangVUSHSDHDE19,DBLP:conf/acl/DaiYYCLS19}. instead of the absolute position encoding. We propose a revised relative positional encoding that uses fewer parameters and performs better.

The second is an empirical finding. The attention distribution of the vanilla Transformer is scaled and smooth. But for NER, a sparse attention is suitable since not all words are necessary to be attended. Given a current word, a few contextual words are enough to judge its label. The smooth attention could include some noisy information. Therefore, we abandon the scale factor of dot-production attention and use an un-scaled and sharp attention.

With the above improvements, we can greatly boost the performance of Transformer encoder for NER.

Other than only using Transformer to model the word-level context, we also tried to apply it as a character encoder to model word representation with character-level information.
The previous work has proved that character encoder is necessary to capture the character-level features and alleviate the out-of-vocabulary (OOV) problem~\cite{lample2016neural,ma2016end,DBLP:journals/tacl/ChiuN16,DBLP:conf/emnlp/XinHMR18}. In NER, CNN is commonly used as the character encoder.
However, we argue that CNN is also not perfect for representing character-level information, because the receptive field of CNN is limited, and the kernel size of the CNN character encoder is usually 3, which means it cannot correctly recognize 2-gram or 4-gram patterns. Although we can deliberately design different kernels, CNN still cannot solve patterns with discontinuous characters, such as ``un..ily'' in  ``unhappily" and ``unnecessarily". Instead, the Transformer-based character encoder shall not only fully make use of the concurrence power of GPUs, but also have the potentiality to recognize different n-grams and even discontinuous patterns. Therefore, in this paper, we also try to use Transformer as the character encoder, and we compare four kinds of character encoders.


In summary, to improve the performance of the Transformer-based model in the NER task, we explicitly utilize the directional relative positional encoding, reduce the number of parameters and sharp the attention distribution. After the adaptation, the performance raises a lot, making our model even performs better than BiLSTM based models. Furthermore, in the six NER datasets, we achieve state-of-the-art performance among models without considering the pre-trained language models or designed features.

\section{Related Work}

\subsection{Neural Architecture for NER}

\citet{collobert2011natural} utilized the Multi-Layer Perceptron (MLP) and CNN to avoid using task-specific features to tackle different sequence labeling tasks, such as Chunking, Part-of-Speech (POS) and NER. In~\cite{huang2015bidirectional}, BiLSTM-CRF was introduced to solve sequence labeling questions. Since then, the BiLSTM has been extensively used in the field of NER~\cite{DBLP:journals/tacl/ChiuN16,DBLP:conf/nlpcc/DongZZHD16,DBLP:conf/nlpcc/YangZLZZS18,ma2016end}.


Despite BiLSTM's great success in the NER task, it has to compute token representations one by one, which massively hinders full exploitation of GPU's parallelism. Therefore, CNN has been proposed by~\cite{strubell2017fast,DBLP:conf/ijcai/GuiM0ZJH19} to encode words concurrently. In order to enlarge the receptive field of CNNs, \citep{strubell2017fast} used iterative dilated CNNs (ID-CNN).





Since the word shape information, such as the capitalization and n-gram, is important in recognizing named entities, CNN and BiLSTM have been used to extract character-level information~\cite{DBLP:journals/tacl/ChiuN16,lample2016neural,ma2016end,strubell2017fast,chen2019grn}. 


Almost all neural-based NER models used pre-trained word embeddings, like Word2vec and Glove~\cite{pennington2014glove,mikolov2013efficient}. And when contextual word embeddings are combined, the performance of NER models will boost a lot~\cite{DBLP:conf/acl/PetersABP17,peters2018deep,DBLP:conf/coling/AkbikBV18}. ELMo introduced by~\cite{peters2018deep} used the CNN character encoder and BiLSTM language models to get contextualized word representations. Except for the BiLSTM based pre-trained models, BERT was based on Transformer~\cite{DBLP:journals/corr/abs-1810-04805}.

\subsection{Transformer}

Transformer was introduced by \citep{vaswani2017attention}, which was mainly based on self-attention. It achieved great success in various NLP tasks. Since the self-attention mechanism used in the Transformer is unaware of positions, to avoid this shortage, position embeddings were used \cite{vaswani2017attention,DBLP:journals/corr/abs-1810-04805}. Instead of using the sinusoidal position embedding \cite{vaswani2017attention} and learned absolute position embedding, \citet{DBLP:conf/naacl/ShawUV18} argued that the distance between two tokens should be considered when calculating their attention score. \citet{DBLP:conf/iclr/HuangVUSHSDHDE19} reduced the computation complexity of relative positional encoding from $O(l^2d)$ to $O(ld)$, where $l$ is the length of sequences and $d$ is the hidden size. \citet{DBLP:conf/acl/DaiYYCLS19} derived a new form of relative positional encodings, so that the relative relation could be better considered.


\subsubsection{Transformer Encoder Architecture}

We first introduce the Transformer encoder proposed in~\cite{vaswani2017attention}. The Transformer encoder takes in an matrix $H \in \mathbb{R}^{l \times d}$, where $l$ is the sequence length, $d$ is the input dimension. Then three learnable matrix $W_q$, $W_k$, $W_v$ are used to project $H$ into different spaces. Usually, the matrix size of the three matrix are all $\mathbb{R}^{d \times d_k}$, where $d_k$ is a hyper-parameter. After that, the scaled dot-product attention can be calculated by the following equations,

\vspace{-1em}
{\small
  \begin{align}
     Q, K, V & = HW_q, HW_k, HW_v, \label{eq:qkv} \\
     A_{t,j} & = Q_tK_j^T, \label{eq:attn_transformer} \\
     \mathrm{Attn}(K, Q, V) & = \mathrm{softmax}(\frac{A}{\sqrt{d_k}})V, \label{eq:softmax_transformer}
  \end{align}
}%
where $Q_t$ is the query vector of the $t$th token, $j$ is the token the $t$th token attends. $K_j$ is the key vector representation of the $j$th token.
The softmax is along the last dimension. Instead of using one group of $W_q$, $W_k$, $W_v$, using several groups will enhance the ability of self-attention. When several groups are used, it is called multi-head self-attention, the calculation can be formulated as follows,

\vspace{-1em}
{\small
\begin{align}
  Q^{(h)}, K^{(h)}, V^{(h)} & = HW_{q}^{(h)}, HW_{k}^{(h)}, HW_{v}^{(h)}, \\
  head^{(h)} & = \mathrm{Attn}(Q^{(h)}, K^{(h)}, V^{(h)}), \\
  \mathrm{MultiHead}(H) & = [head^{(1)}; ...; head^{(n)}]W_O, \label{eq:multi-head}
\end{align}
}%
where $n$ is the number of heads, the superscript $h$ represents the head index. $[head^{(1)}; ...; head^{(n)}]$ means concatenation in the last dimension. Usually $d_k \times n = d$, which means the output of $[head^{(1)}; ...; head^{(n)}]$ will be of size $\mathbb{R}^{l \times d}$. $W_o$ is a learnable parameter, which is of size $\mathbb{R}^{d \times d}$.

The output of the multi-head attention will be further processed by the position-wise feed-forward networks, which can be represented as follows,

\vspace{-1em}
{\small
\begin{align}
  \mathrm{FFN}(x) = max(0, xW_1 + b_1)W_2 + b_2, \label{eq:ffn}
\end{align}
}%
where $W_1$, $W_2$, $b_1$, $b_2$ are learnable parameters, and $W_1 \in \mathbb{R}^{d \times d_{ff}}$, $W_2 \in \mathbb{R}^{d_{ff} \times d}$, $b_1 \in \mathbb{R}^{d_{ff}}$, $b_2 \in \mathbb{R}^{d}$. $d_{ff}$ is a hyper-parameter. Other components of the Transformer encoder includes layer normalization and Residual connection, we use them the same as \cite{vaswani2017attention}.

\subsubsection{Position Embedding} \label{sec:position_embedding}

The self-attention is not aware of the positions of different tokens, making it unable to capture the sequential characteristic of languages. In order to solve this problem, \citep{vaswani2017attention} suggested to use position embeddings generated by sinusoids of varying frequency. The $t$th token's position embedding can be represented by the following equations

\vspace{-1em}
{\small
  \begin{align}
    PE_{t,2i} & = \sin(t/10000^{2i/d}), \label{eq:position_embedding1} \\
    PE_{t,2i+1} & = \cos(t/10000^{2i/d}),  \label{eq:position_embedding2}
  \end{align}
}%
where $i$ is in the range of $[0, \frac{d}{2}]$, $d$ is the input dimension. This sinusoid based position embedding makes Transformer have an ability to model the position of a token and the distance of each two tokens. For any fixed offset $k$, $PE_{t+k}$ can be represented by a linear transformation of $PE_{t}$~\cite{vaswani2017attention}.

\section{Proposed Model}

In this paper, we utilize the Transformer encoder to model the long-range and complicated interactions of sentence for NER.
The structure of proposed model is shown in Fig~\ref{fig:whole_model_structure}. We detail each parts in the following sections.

\begin{figure}[t]
    \centering
    \includegraphics[width=0.45\textwidth]{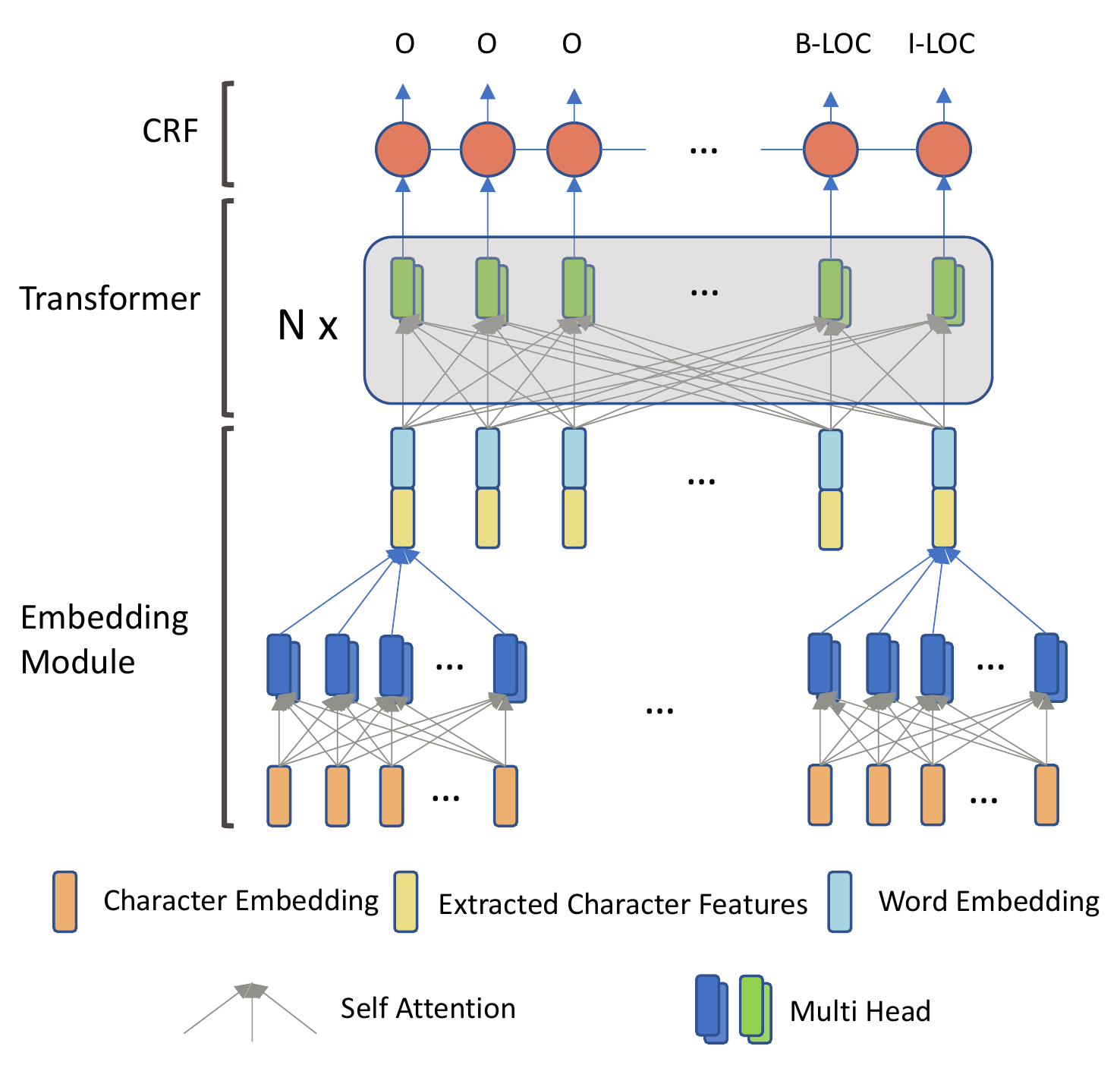}
    \caption{Model structure of TENER for English NER tasks. In TENER, Transformer encoder is used not only to extract the word-level contextual information, but also to encode character-level information in a word.} \label{fig:whole_model_structure}
\end{figure}

\subsection{Embedding Layer}


To alleviate the problems of data sparsity and out-of-vocabulary (OOV), most NER models adopted the CNN character encoder \cite{ma2016end,ye2018hybrid,chen2019grn} to represent words.
Compared to BiLSTM based character encoder~\citep{lample2016neural,ghaddar2018robust}, CNN is more efficient. Since Transformer can also fully exploit the GPU's parallelism, it is interesting to use Transformer as the character encoder. A potential benefit of Transformer-based character encoder is to extract different n-grams and even uncontinuous character patterns, like ``un..ily'' in ``unhappily'' and ``uneasily''. For the model's uniformity, we use the ``adapted Transformer'' to represent the Transformer introduced in next subsection.

The final word embedding is the concatenation of the character features extracted by the character encoder and the pre-trained word embeddings.

\subsection{Encoding Layer with Adapted Transformer}


Although Transformer encoder has potential advantage in modeling long-range context, it is not working well for NER task. In this paper, we propose an adapted Transformer for NER task with two improvements.

\subsubsection{Direction- and Distance-Aware Attention}

Inspired by the success of BiLSTM in NER tasks, we consider what properties the Transformer lacks compared to BiLSTM-based models. One observation is that BiLSTM can discriminatively collect the context information of a token from its left and right sides. But it is not easy for the Transformer to distinguish which side the context information comes from.

Although the dot product between two sinusoidal position embeddings is able to reflect their distance, it lacks directionality and this property will be broken by the vanilla Transformer attention. To illustrate this, we first prove two properties of the sinusoidal position embeddings.

\begin{prop}
  For an offset $k$ and a position $t$, $PE_{t+k}^TPE_{t}$ only depends on $k$, which means the dot product of two sinusoidal position embeddings can reflect the distance between two tokens.
\end{prop}
\begin{proof}
  Based on the definitions of Eq.\eqref{eq:position_embedding1} and Eq.\eqref{eq:position_embedding2},
the position embedding of $t$-th token is
\begin{seequation}
  PE_t  = \left[ \begin{array}{c}
    \sin(c_0t)\\
    \cos(c_0t)\\
    \vdots \\
    \sin(c_{\frac{d}{2}-1}t)\\
    \cos(c_{\frac{d}{2}-1}t) \\
  \end{array} \right],
\end{seequation}
where $d$ is the dimension of the position embedding, $c_i$ is a constant decided by $i$, and its value is $1/10000^{2i/d}$.

Therefore,

 \vspace{-1em}
{\small
\begin{align}
 PE_t^TPE_{t+k} & = \sum_{\substack{j=0}}^{\frac{d}{2}-1} [\sin(c_jt)\sin(c_j(t+k))\nonumber \\
                & \qquad + \cos(c_jt)\cos(c_j(t+k))] \label{eq:cos_before}\\
               & = \sum_{\substack{j=0}}^{\frac{d}{2}-1} \cos(c_j(t-(t+k))) \label{eq:cos_after}\\
               & = \sum_{\substack{j=0}}^{\frac{d}{2}-1} \cos(c_jk),
\end{align}
}%
where Eq.\eqref{eq:cos_before} to Eq.\eqref{eq:cos_after} is based on the equation $\cos(x-y) = \sin(x)\sin(y) + \cos(x)\cos(y)$.
\end{proof}

\begin{figure}[t]
    \centering
    \includegraphics[width=0.45\textwidth]{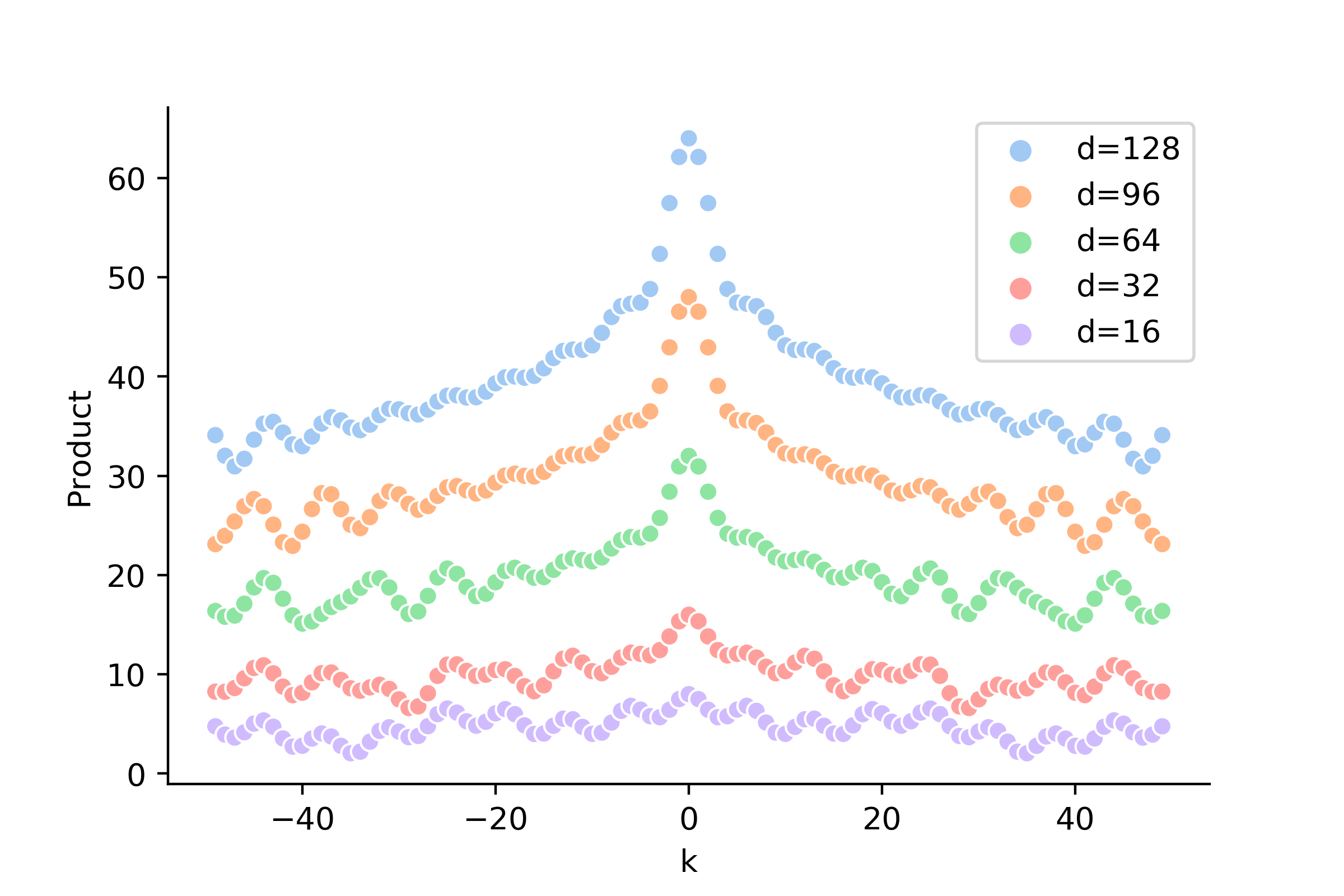}
    \caption{Dot product between two sinusoidal position embeddings whose distance is $k$. It is clear that the product is symmetrical, and with the increment of $|k|$, it has a trend to decrease, but this decrease is not monotonous. }\label{fig:position_curve}
\end{figure}

\begin{figure}[t]
    \centering
    \includegraphics[width=0.45\textwidth]{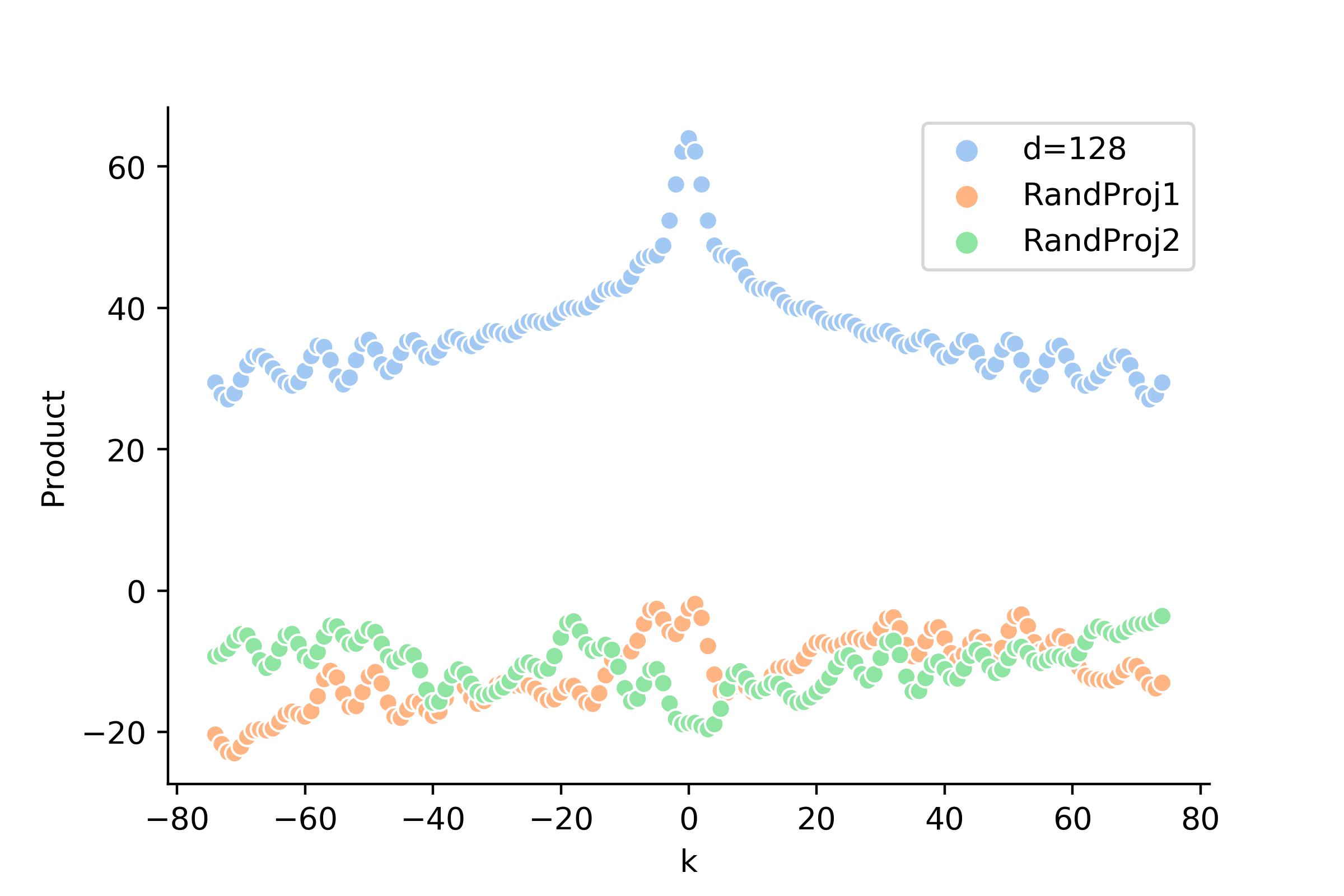}
    \caption{The upper line is the product between $PE_t^T PE_{t+k}$. The lower two lines are the products of $PE_t^T W PE_{t+k}$ with two random $W$s. Although $PE_t^TPE_{t+k}$ can reflect the distance, the $PE_t^TWPE_{t+k}$ has no clear pattern. }\label{fig:position_curve_2}
\end{figure}

\begin{prop}
For an offset $k$ and a position $t$, $PE_{t}^TPE_{t-k}=PE_{t}^TPE_{t+k}$, which means the sinusoidal position embeddings is unware of directionality.
\end{prop}
\begin{proof}

Let $j=t-k$, according to property 1, we have

\vspace{-1em}
{\small
\begin{align}
 PE_{t}^TPE_{t+k} &= PE_{j}^TPE_{j+k}\\
 &=PE_{t-k}^TPE_{t}.
\end{align}
}%
\end{proof}

The relation between $d$, $k$ and $PE_t^TPE_{t+k}$ is displayed in Fig \ref{fig:position_curve}. The sinusoidal position embeddings are distance-aware but lacks directionality.

However, the property of distance-awareness also disappears when $PE_t$ is projected into the query and key space of self-attention. Since in vanilla Transformer the calculation between $PE_t$ and $PE_{t+k}$ is actually $PE_t^TW_q^TW_kPE_{t+k}$, where $W_q, W_k$ are parameters in Eq.\eqref{eq:qkv}. Mathematically, it can be viewed as $PE_t^TWPE_{t+k}$ with only one parameter $W$. The relation between $PE_t^TPE_{t+k}$ and $PE_t^TWPE_{t+k}$ is depicted in Fig \ref{fig:position_curve_2}.

Therefore, to improve the Transformer with direction- and distance-aware characteristic, we calculate the attention scores using the equations below:

\vspace{-1em}
{\small
\begin{align}
  & Q, K, V  =  HW_q, H_{d_k}, HW_v, \label{eq:project} \\
  & R_{t-j}\!= [\ldots \,  \sin(\frac{t-j}{10000^{2i/d_k}}) \,  \cos(\frac{t-j}{10000^{2i/d_k}}) \, \ldots]^T,  \label{eq:relative}\\
  & A^{rel}_{t,j} = Q_tK_j^T + Q_tR_{t-j}^T + \mathbf{u}K_j^T + \mathbf{v}R_{t-j}^T, \label{eq:attn_xl} \\
  & \mathrm{Attn}(Q, K, V) = \mathrm{softmax}(A^{rel})V, \label{eq:softmax_xl}
\end{align}}%
where $t$ is index of the target token, $j$ is the index of the context token, $Q_t, K_j$ is the query vector and key vector of token $t, j$ respectively, $W_q, W_v \in \mathbb{R}^{d \times d_k}$. To get $H_{d_k}\in \mathbb{R}^{l \times d_k}$, we first split $H$ into $d/d_k$ partitions in the second dimension, then for each head we use one partition. $\mathbf{u} \in \mathbb{R}^{d_k}$, $\mathbf{v} \in \mathbb{R}^{d_k}$ are learnable parameters,
$R_{t-j}$ is the relative positional encoding, and $R_{t-j} \in \mathbb{R}^{d_k}$, $i$ in Eq.\eqref{eq:relative} is in the range $[0, \frac{d_k}{2}]$. $Q_t^TK_j$ in Eq.\eqref{eq:attn_xl} is the attention score between two tokens; $Q_t^TR_{t-j}$ is the $t$th token's bias on certain relative distance; $u^TK_j$ is the bias on the $j$th token; $v^TR_{t-j}$ is the bias term for certain distance and direction.

Based on Eq.\eqref{eq:relative}, we have

\vspace{-1em}
{\small
\begin{align}
  R_{t}, R_{-t} = \left[ \begin{array}{c}
    \sin(c_0t)\\
    \cos(c_0t)\\
    \vdots \\
    \sin(c_{\frac{d}{2}-1}t)\\
    \cos(c_{\frac{d}{2}-1}t) \\
  \end{array} \right],
  \left[ \begin{array}{c}
    -\sin(c_0t)\\
    \cos(c_0t)\\
    \vdots \\
    -\sin(c_{\frac{d}{2}-1}t)\\
    \cos(c_{\frac{d}{2}-1}t) \\
  \end{array} \right],
\end{align}}%
because $\sin(-x)=-\sin(x), \cos(x)=\cos(-x)$. This means for an offset $t$, the forward and backward relative positional encoding are the same with respect to the $\cos(c_it)$ terms, but is the opposite with respect to the $\sin(c_it)$ terms. Therefore, by using $R_{t-j}$, the attention score can distinguish different directions and distances.

The above improvement is based on the work~\citep{DBLP:conf/naacl/ShawUV18,DBLP:conf/acl/DaiYYCLS19}.
Since the size of NER datasets is usually small, we avoid direct multiplication of two learnable parameters, because they can be represented by one learnable parameter. Therefore we do not use $W_k$ in Eq.\eqref{eq:project}. The multi-head version is the same as Eq.\eqref{eq:multi-head}, but we discard $W_o$ since it is directly multiplied by $W_1$ in Eq.\eqref{eq:ffn}.


\subsubsection{Un-scaled Dot-Product Attention}
The vanilla Transformer use the scaled dot-product attention to smooth the output of softmax function. In Eq.\eqref{eq:softmax_transformer}, the dot product of key and value matrices is divided by the scaling factor $\sqrt{d_k}$.

We empirically found that models perform better without the scaling factor $\sqrt{d_k}$. We presume this is because without the scaling factor the attention will be sharper. And the sharper attention might be beneficial in the NER task since only few words in the sentence are named entities.

\subsection{CRF Layer}
In order to take advantage of dependency between different tags, the Conditional Random Field (CRF) was used in all of our models. Given a sequence $\mathbf{s}=[s_1, s_2, ..., s_T]$, the corresponding golden label sequence is $\mathbf{y}=[y_1, y_2, ..., y_T]$, and $\mathbf{Y}(\mathbf{s})$ represents all valid label sequences. The probability of $\mathbf{y}$ is calculated by the following equation

\vspace{-1em}
{\small
\begin{align}
  P(\mathbf{y}|\mathbf{s}) = \frac{\sum_{t=1}^{T}e^{f(\mathbf{y}_{t-1},\mathbf{y}_t,\mathbf{s})}}
  {\sum_{\mathbf{y}^{\prime}}^{\mathbf{Y}(\mathbf{s})}\sum_{t=1}^{T}e^{f(\mathbf{y}_{t-1}^{\prime},\mathbf{y}_t^{\prime},\mathbf{s})}},
\end{align}}%
where $f(\mathbf{y}_{t-1},\mathbf{y}_t,\mathbf{s})$ computes the transition score from $\mathbf{y}_{t-1}$ to $\mathbf{y}_t$ and the score for $\mathbf{y}_t$. The optimization target is to maximize $P(\mathbf{y}|\mathbf{s})$. When decoding, the Viterbi Algorithm is used to find the path achieves the maximum probability.

\section{Experiment}

\subsection{Data}

We evaluate our model in two English NER datasets and four Chinese NER datasets.

\begin{table}[t]\small \setlength{\tabcolsep}{2pt}
  \centering
    \begin{threeparttable}
    \caption{Details of Datasets.} \label{tab:datasets_info}
    \begin{tabular}{llcrrr}
    \toprule
                             & Dataset                        & Type     & Train   & Dev    & Test   \\ \midrule
    \multirow{4}{*}{English} & \multirow{2}{*}{CoNLL2003}    & Sentence & 14.0k   & 3.2k   & 3.5k   \\
                             &                                & Token    & 203.6k  & 51.4k  & 46.4k  \\ \cline{2-6}
                             & \multirow{2}{*}{OntoNotes 5.0} & Sentence & 59.9k   & 8.5k   & 8.3k   \\
                             &                                & Token    & 1088.5k & 147.7k & 152.7k \\ \hline
    \multirow{8}{*}{Chinese} & \multirow{2}{*}{OntoNotes 4.0} & Sentence & 15.7k   & 4.3k   & 4.3k   \\
                             &                                & Token    & 491.9k  & 200.5k & 208.1k \\ \cline{2-6}
                             & \multirow{2}{*}{MSRA}          & Sentence & 46.4k   & 4.4k   & 4.4k   \\
                             &                                & Token    & 2169.9k & 172.6k & 172.6k \\ \cline{2-6}
                             & \multirow{2}{*}{Weibo}         & Sentence & 1.4k    & 0.3k   & 0.3k   \\
                             &                                & Token    & 73.5k   & 14.4k  & 14.8k  \\ \cline{2-6}
                             & \multirow{2}{*}{Resume}        & Sentence & 3.8k    & 0.5k   & 0.5k   \\
                             &                                & Token    & 124.1k  & 13.9k  & 15.1k  \\ \bottomrule
    \end{tabular}
\end{threeparttable}
\end{table}
\begin{table*}[t]\small 
  \centering
  \begin{tabular}{lllll}
  \toprule
  	 Models            & Weibo  & Resume    & OntoNotes4.0    & MSRA             \\ \hline
     BiLSTM $^{\clubsuit}$ & 56.75      & 94.41             & 71.81            & 91.87            \\
     ID-CNN $^{\spadesuit}$ & -                 & 93.75             & 62.25            & -                \\
     CAN-NER$^*$ \citep{DBLP:conf/naacl/ZhuW19} & 59.31             & 94.94             & 73.64            & 92.97            \\
     Transformer      & 46.38 $\pm$ 0.78  & 93.43 $\pm$ 0.26  & 66.49 $\pm$  0.30 & 88.35 $\pm$ 0.60 \\
     TENER(Ours)  & \textbf{58.17 $\pm$ 0.22}  & \textbf{95.00 $\pm$ 0.25}  & \textbf{72.43 $\pm$ 0.39} & \textbf{92.74 $\pm$ 0.27} \\
     \quad w/ scale    & 57.40 $\pm$ 0.3  & 94.00 $\pm$ 0.51 & 71.72 $\pm$ 0.08 & 91.67 $\pm$ 0.23\\
     \bottomrule
  \end{tabular}
\caption{The F1 scores on Chinese NER datasets. $^{\clubsuit}$,$^{\spadesuit}$ are results reported in \citep{DBLP:conf/acl/ZhangY18} and \citep{DBLP:conf/ijcai/GuiM0ZJH19}, respectively. ``w/ scale" means TENER using the scaled attention in Eq.\eqref{eq:softmax_xl}.
$^*$ their results are not directly comparable with ours, since they used 100d pre-trained character and bigram embeddings. Other models use the same embeddings.
} \label{tab:cn_ner}
\end{table*}

(1) CoNLL2003 is one of the most evaluated English NER datasets, which contains four different named entities: PERSON, LOCATION, ORGANIZATION, and MISC \cite{DBLP:conf/conll/SangM03}.

(2) OntoNotes 5.0 is an English NER dataset whose corpus comes from different domains, such as telephone conversation, newswire. We exclude the New Testaments portion since there is no named entity in it \cite{chen2019grn,DBLP:journals/tacl/ChiuN16}. This dataset has eleven entity names and seven value types, like CARDINAL, MONEY, LOC.

(3) \citet{Weischedel2011ontonotes} released OntoNotes 4.0. In this paper, we use the Chinese part. We adopted the same pre-process as \cite{DBLP:conf/naacl/CheWML13}.

(4) The corpus of the Chinese NER dataset MSRA came from news domain \cite{DBLP:conf/acl-sighan/Levow06}.

(5) Weibo NER was built based on text in Chinese social media Sina Weibo \cite{DBLP:conf/emnlp/PengD15}, and it contained 4 kinds of entities.

(6) Resume NER was annotated by \cite{DBLP:conf/acl/ZhangY18}.

Their statistics are listed in Table \ref{tab:datasets_info}. For all datasets, we replace all digits with ``0'', and use the BIOES tag schema. For English, we use the Glove 100d pre-trained embedding~\cite{pennington2014glove}. For the character encoder, we use 30d randomly initialized character embeddings. More details on models' hyper-parameters can be found in the supplementary material. For Chinese, we used the character embedding and bigram embedding released by \cite{DBLP:conf/acl/ZhangY18}. All pre-trained embeddings are finetuned during training. In order to reduce the impact of randomness, we ran all of our experiments at least three times, and its average F1 score and standard deviation are reported.

We used random-search to find the optimal hyper-parameters, hyper-parameters and their ranges are displayed in the supplemental material. We use SGD and 0.9 momentum to optimize the model. We run 100 epochs and each batch has 16 samples. During the optimization, we use the triangle learning rate~\cite{DBLP:conf/wacv/Smith17} where the learning rate rises to the pre-set learning rate at the first 1\% steps and decreases to 0 in the left 99\% steps. The model achieves the highest development performance was used to evaluate the test set. The hyper-parameter search range and other settings can be found in the supplementary material. Codes are available at \url{https://github.com/fastnlp/TENER}.

\subsection{Results on Chinese NER Datasets}

We first present our results in the four Chinese NER datasets. Since Chinese NER is directly based on the characters, it is more straightforward to show the abilities of different models without considering the influence of word representation.

As shown in Table \ref{tab:cn_ner}, the vanilla Transformer does not perform well and is worse than the BiLSTM and CNN based models. However, when relative positional encoding combined, the performance was enhanced greatly, resulting in better results than the BiLSTM and CNN in all datasets. The number of training examples of the Weibo dataset is tiny, therefore the performance of the Transformer is abysmal, which is as expected since the Transformer is data-hungry. Nevertheless, when enhanced with the relative positional encoding and unscaled attention, it can achieve even better performance than the BiLSTM-based model. The superior performance of the adapted Transformer in four datasets ranging from small datasets to big datasets depicts that the adapted Transformer is more robust to the number of training examples than the vanilla Transformer. As the last line of Table \ref{tab:cn_ner} depicts, the scaled attention will deteriorate the performance.

\subsection{Results on English NER datasets}

\begin{table}[t]\small \setlength{\tabcolsep}{3pt}
  \centering
  \begin{threeparttable}
  \begin{tabular}{p{12em}ll}
  \toprule
  Models                        & CoNLL2003        & OntoNotes 5.0    \\ \midrule
  BiLSTM-CRF \citep{huang2015bidirectional} & 88.83            &                  \\
  CNN-BiLSTM-CRF \citep{DBLP:journals/tacl/ChiuN16}          & 90.91 $\pm$ 0.20 & 86.12 $\pm$ 0.22 \\
  BiLSTM-BiLSTM-CRF \citep{lample2016neural}       & 90.94            &                  \\
  CNN-BiLSTM-CRF \citep{ma2016end}             & 91.21            &                  \\
  ID-CNN \citep{strubell2017fast}       & 90.54 $\pm$ 0.18 & 86.84 $\pm$ 0.19 \\
  LM-LSTM-CRF \citep{DBLP:conf/aaai/LiuSRXG0018}         & 91.24 $\pm$ 0.12 &                  \\
  CRF+HSCRF \citep{ye2018hybrid}           & 91.26 $\pm$ 0.1  &                  \\
  BiLSTM-BiLSTM-CRF \citep{adnan2018}              & 91.11            &             \\
  LS+BiLSTM-CRF \citep{ghaddar2018robust}      & 90.52 $\pm$ 0.20 & 86.57 $\pm$ 0.1  \\
  CN$^3$ \citep{liu2019contextualized}  & 91.1             &                  \\
  GRN \citep{chen2019grn}           & 91.44 $\pm$ 0.16 & 87.67 $\pm$ 0.17 \\
  Transformer                   & 89.57 $\pm$ 0.12 &    86.73 $\pm$ 0.07  \\
  TENER (Ours)               & 91.33 $\pm$ 0.05 & \textbf{88.43 $\pm$ 0.12} \\
  \quad w/ scale                        & 91.06 $\pm$ 0.09     & 87.94 $\pm$ 0.1 \\
  \quad w/ CNN-char              & \textbf{91.45 $\pm$ 0.07} & 88.25 $\pm$ 0.11 \\
  \bottomrule
  \end{tabular}
  \end{threeparttable}
  \caption{The F1 scores on English NER datasets. We only list results based on non-contextualized embeddings, and methods utilized pre-trained language models, pre-trained features, or higher dimension word vectors are excluded. TENER (Ours) uses the Transformer encoder both in the character-level and word-level. ``w/ scale" means TENER using the scaled attention in Eq.\eqref{eq:softmax_xl}.  ``w/ CNN-char'' means TENER using CNN as character encoder instead of AdaTrans.
  } \label{tab:en_ner}
\end{table}

\begin{table}[h]\small \setlength{\tabcolsep}{3pt}
  \begin{threeparttable}
  \begin{tabular}{p{12em}ll}
    \toprule
      Models                & CoNLL2003      & OntoNotes 5.0    \\ \midrule
      BiLSTM                & 92.55$\pm$0.10 & 88.88$\pm$0.16  \\
      GRN \citep{chen2019grn} & 92.34$\pm$0.1  & -               \\
      TENER (Ours)                & \textbf{92.62$\pm$0.09} & \textbf{89.78$\pm$0.15} \\
    \bottomrule
  \end{tabular}
  \end{threeparttable}
  \caption{Performance of models with ELMo as their embeddings in English NER datasets. ``BiLSTM" is our run. In the larger OntoNotes5.0, TENER achieves much better F1 score.} \label{tab:en_ner_elmo}
\end{table}

\begin{table}[ht ]\small \setlength{\tabcolsep}{3pt}
  \centering
\subfloat[CoNLL2003]{\label{exp3:conll}
\begin{tabular}{llll}
\toprule
 \diagbox[innerwidth = 2em, width = 5em, height = 4ex]{Char}{Word}               & BiLSTM & ID-CNN & AdaTrans\\ \midrule
No Char         & 88.34 $\pm$ 0.32  & 87.30  $\pm$ 0.15  &  88.37 $\pm$ 0.27    \\
BiLSTM          & 91.32 $\pm$ 0.13  & 89.99 $\pm$ 0.14  &  91.29 $\pm$ 0.12     \\
CNN             & 91.22 $\pm$ 0.10  & 90.17 $\pm$ 0.02  &  \textbf{91.45 $\pm$ 0.07}\\
Transformer     & 91.12 $\pm$ 0.10  & 90.05 $\pm$ 0.13  &  91.23 $\pm$ 0.06     \\
AdaTrans & 91.38 $\pm$ 0.15  & 89.99 $\pm$ 0.05  &  91.33 $\pm$ 0.05 \\\bottomrule
\end{tabular}
}
\\

\subfloat[OntoNotes 5.0]{\label{exp3:ontonotes}
\begin{tabular}{llll}
\toprule
 \diagbox[innerwidth = 2em, width = 5em, height = 4ex]{Char}{Word}                 & BiLSTM & ID-CNN & AdaTrans \\ \midrule
No Char         & 85.20 $\pm$ 0.23  & 84.26 $\pm$ 0.07  &  85.80 $\pm$ 0.10  \\
BiLSTM          & 87.85 $\pm$ 0.09  & 87.38 $\pm$ 0.17  &  88.12 $\pm$ 0.16    \\
CNN             & 87.79 $\pm$ 0.14  & 87.10 $\pm$ 0.06  &  88.25 $\pm$ 0.11   \\
Transformer     & 88.01 $\pm$ 0.06  & 87.31 $\pm$ 0.10  &  88.20 $\pm$ 0.07     \\
AdaTrans & 88.12 $\pm$ 0.17  & 87.51 $\pm$ 0.11  &  \textbf{88.43 $\pm$ 0.12}    \\ \bottomrule
\end{tabular}
}
\caption{F1 scores in the CoNLL2003 and OntoNotes 5.0. ``Char'' means character-level encoder, and ``Word'' means word-level encoder. ``AdaTrans'' means our adapted Transformer encoder.}
\end{table}

The comparison between different NER models on English NER datasets is shown in Table \ref{tab:en_ner}. The poor performance of the Transformer in the NER datasets was also reported by \cite{DBLP:conf/naacl/GuoQLSXZ19}. Although performance of the Transformer is higher than \cite{DBLP:conf/naacl/GuoQLSXZ19}, it still lags behind the BiLSTM-based models~\cite{ma2016end}. Nonetheless, the performance is massively enhanced by incorporating the relative positional encoding and unscaled attention into the Transformer. The adaptation not only makes the Transformer achieve superior performance than BiLSTM based models, but also unveil the new state-of-the-art performance in two NER datasets when only the Glove 100d embedding and CNN character embedding are used. The same deterioration of performance was observed when using the scaled attention. Besides, if ELMo was used \cite{peters2018deep}, the performance of TENER can be further boosted as depicted in Table \ref{tab:en_ner_elmo}.


\subsection{Analysis of Different Character Encoders}

The character-level encoder has been widely used in the English NER task to alleviate the data sparsity and OOV problem in word representation. In this section, we cross different character-level encoders (BiLSTM, CNN, Transformer encoder and our adapted Transformer encoder (AdaTrans for short) ) and different word-level encoders (BiLSTM, ID-CNN and AdaTrans) to implement the NER task. Results on CoNLL2003 and OntoNotes 5.0 are presented in Table \ref{exp3:conll} and Table \ref{exp3:ontonotes}, respectively.

The ID-CNN encoder is from~\cite{strubell2017fast}, and we re-implement their model in PyTorch. For different combinations, we use random search to find its best hyper-parameters. Hyper-parameters for character encoders were fixed. The details can be found in the supplementary material.

For the results on CoNLL2003 dataset which is depicted in Table \ref{exp3:conll}, the AdaTrans performs as good as the BiLSTM in different character encoder scenario averagely. In addition, from Table \ref{exp3:ontonotes}, we can find the pattern that the AdaTrans character encoder outpaces the BiLSTM and CNN character encoders when different word-level encoders being used. Moreover, no matter what character encoder being used or none being used, the AdaTrans word-level encoder gets the best performance. This implies that when the number of training examples increases, the AdaTrans character-level and word-level encoder can better realize their ability.

\subsection{Convergent Speed Comparison}

\begin{figure}[ht]
    \centering
    \includegraphics[width=0.45\textwidth]{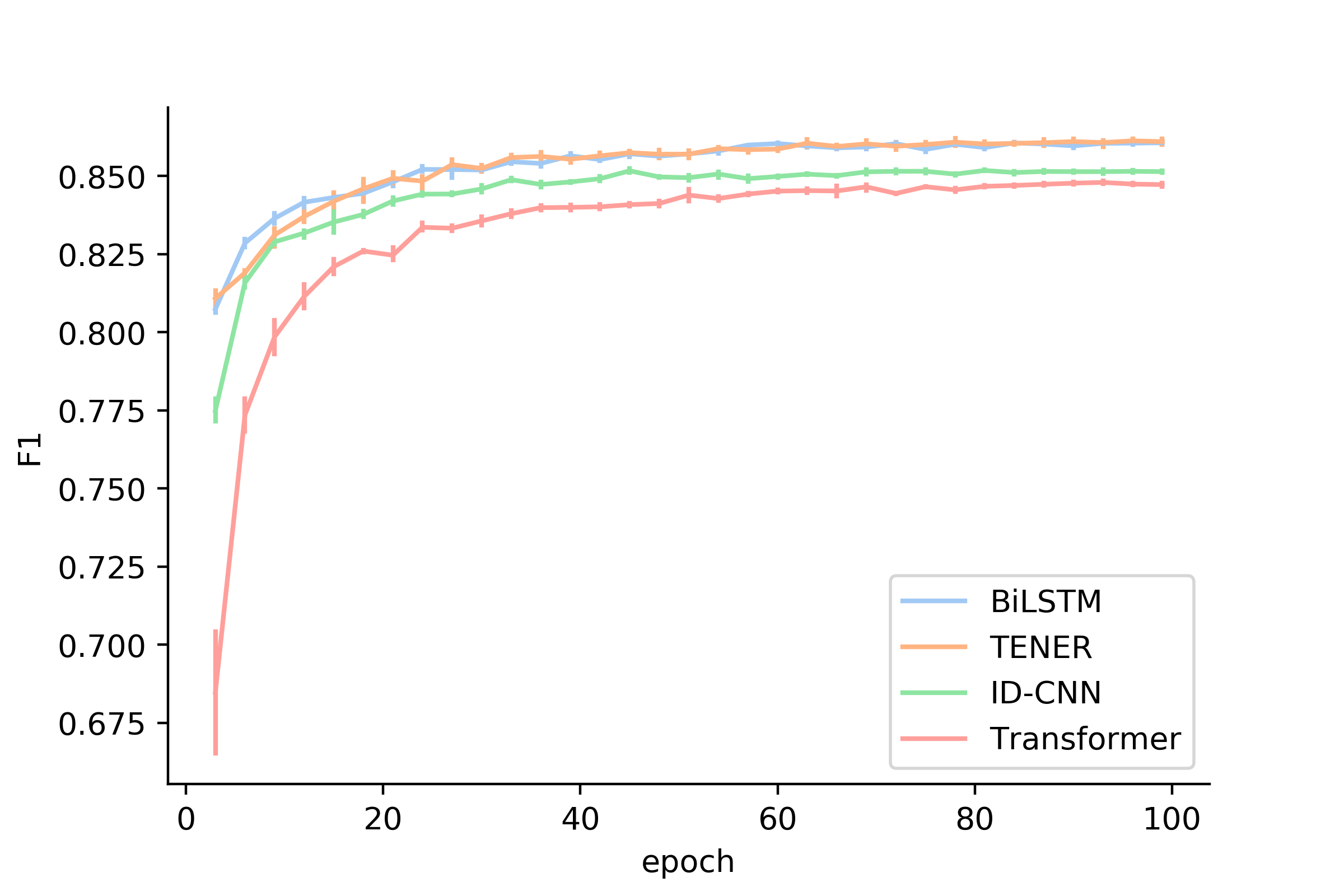}
    \caption{Convergent speed in the development dataset of OntoNotes 5.0 for four kinds of models.} \label{fig:converge_speed}
\end{figure}

We compare the convergent speed of BiLSTM, ID-CNN, Transformer, and TENER in the development set of the OntoNotes 5.0. The curves are shown in Fig \ref{fig:converge_speed}. TENER converges as fast as the BiLSTM model and outperforms the vanilla Transformer.

\section{Conclusion}
In this paper, we propose TENER, a model adopting Transformer Encoder with specific customizations for the NER task. Transformer Encoder has a powerful ability to capture the long-range context. In order to make the Transformer more suitable to the NER task, we introduce the direction-aware, distance-aware and un-scaled attention. Experiments in two English NER tasks and four Chinese NER tasks show that the performance can be massively increased. Under the same pre-trained embeddings and external knowledge, our proposed modification outperforms previous models in the six datasets. Meanwhile, we also found the adapted Transformer is suitable for being used as the English character encoder, because it has the potentiality to extract intricate patterns from characters. Experiments in two English NER datasets show that the adapted Transformer character encoder performs better than BiLSTM and CNN character encoders.

%

\bibliography{nlp}
\bibliographystyle{acl_natbib}

\section{Supplemental Material}

\subsection{Character Encoder} \label{supply:char_embed}

We exploit four kinds of character encoders. For all character encoders, the randomly initialized character embeddings are 30d. The hidden size of BiLSTM used in the character encoder is 50d in each direction. The kernel size of CNN used in the character encoder is 3, and we used 30 kernels with stride 1. For Transformer and adapted Transformer, the number of heads is 3, and every head is 10d, the dropout rate is 0.15, the feed-forward dimension is 60. The Transformer used the sinusoid position embedding. The number of parameters for the character encoder (excluding character embedding) when using BiLSTM, CNN, Transformer and adapted Transformer are 35830, 3660, 8460 and 6600 respectively. For all experiments, the hyper-parameters of character encoders stay unchanged.

\subsection{Hyper-parameters} \label{supply:random_search}

The hyper-parameters and search ranges for different encoders are presented in Table \ref{tb:bilstm}, Table \ref{tb:dilate-cnn} and Table \ref{tb:rel-transformer}.

\begin{table}[t] \small
  \begin{tabular}{lc}\toprule
                   & English                    \\ \midrule
  number of layers & [1, 2]                     \\
  hidden size      & [200, 400, 600, 800, 1200] \\
  learning rate    & [0.01, 0.007, 0.005]       \\
  fc dropout       & 0.4 \\ \bottomrule
  \end{tabular}
  \caption{The hyper-parameters and hyper-parameter search ranges for BiLSTM.} \label{tb:bilstm}
\end{table}

\begin{table}[t] \small
  \begin{tabular}{lc}
  \toprule
                    & English                        \\
  \midrule
  number of layers  & [2, 3, 4, 5, 6]                \\
  number of kernels & [200, 400, 600, 800]           \\
  kernel size       & 3                              \\
  learning rate     & [2e-3, 1.5e-3, 1e-3, 7e-4] \\
  fc dropout        & 0.4  \\ \bottomrule
  \end{tabular}
  \caption{The hyper-parameters and hyper-parameter search ranges for ID-CNN.} \label{tb:dilate-cnn}
\end{table}

\begin{table}[]\small \setlength{\tabcolsep}{3pt}
  \centering
  \begin{tabular}{lcc} \toprule
                      & Chinese                 & English                  \\ \midrule
  number of layers    & [1, 2]                  & [1, 2]                   \\
  number of head      & [4, 6, 8, 10]           & [8, 10, 12, 14]          \\
  head dimension      & [32, 48, 64, 80, 96]    & [64, 80, 96, 112, 128]   \\
  learning rate       & [1e-3, 5e-4, 7e-4] & [9e-4, 7e-4, 5e-4] \\
  transformer dropout & 0.15                    & 0.15                     \\
  fc dropout          & 0.4                     & 0.4 \\ \bottomrule
  \end{tabular}
  \caption{The hyper-parameters and hyper-parameter search ranges for Transformer and adapted Transformer in Chinese and English NER datasets.} \label{tb:rel-transformer}
\end{table}

\end{document}